\documentclass{article}

\usepackage{amssymb}
\usepackage{dsfont}
\usepackage{amsfonts}
\usepackage{microtype}
\usepackage{graphicx}
\usepackage{subfigure}
\usepackage{booktabs}% for professional tables
\usepackage{amsmath}  
\usepackage{bm}
\DeclareMathOperator*{\argmax}{arg\,max}
\newtheorem{theorem}{Theorem}
\newtheorem{lemma}{Lemma}   
\newenvironment{sequation}{\begin{equation}\footnotesize}{\end{equation}}

\newenvironment{proof}{{\noindent\it Proof}\quad}{\hfill $\square$\par}

\usepackage{hyperref}

\usepackage[accepted]{icml2018}

\icmltitlerunning{Efficient Model-Free Reinforcement Learning Using Gaussian Process }

\begin{document}

\twocolumn[
\icmltitle{Efficient Model-Free Reinforcement Learning Using Gaussian Process}



\icmlsetsymbol{equal}{*}

\begin{icmlauthorlist}
\icmlauthor{Ying Fan}{PKUeecs}
\icmlauthor{Letian Chen}{PKUeecs}
\icmlauthor{Yizhou Wang}{PKUeecs}
\end{icmlauthorlist}

\icmlaffiliation{PKUeecs}{Nat'l Engineering Laboratory for Video Technology
Cooperative Medianet Innovation Center
Key Laboratory of Machine Perception (MoE)
Sch'l of EECS, Peking University, Beijing, 100871, China}

\icmlcorrespondingauthor{Yizhou Wang}{yizhou.wang@pku.edu.cn}

\icmlkeywords{Machine Learning, ICML}

\vskip 0.3in
]



\printAffiliationsAndNotice{\icmlEqualContribution} % otherwise use the standard text.

\begin{abstract}
Efficient Reinforcement Learning usually takes advantage of demonstration or good exploration strategy. By applying posterior sampling in model-free RL under the hypothesis of GP, we propose GPPSTD algorithm in continuous state space, giving theoretical justifications and empirical results. We also provide theoretical and empirical results that various demonstration could lower expected uncertainty and benefit posterior sampling exploration. In this way, we combined the demonstration and exploration process together to achieve a more efficient reinforcement learning. 

\end{abstract}
\section{Introduction}
Over the past years, Reinforcement Learning (RL) has achieved a great success in tasks such as Atari Games \cite{mnih2015human}, Go \cite{silver2016mastering}, robot control  \cite{levine2016end} and high-level decisions \cite{silver2013concurrent}. But in general, the conventional RL approaches can hardly obtain a good performance before a large number of experiences are collected. Therefore, two types of methods have been proposed to realize sample efficient learning, i.e. leveraging human demonstration (e.g. inverse RL \cite{ng2000algorithms}) and designing better exploration strategies. Although the literature has plenty of interesting studies on either one, there seems lack of work combining them to our best knowledge. In this paper we propose a new model-free exploration strategy which leverages all kinds of demonstrations (even including unsuccessful ones) to improve learning efficiency. 

Existing works on learning from demonstration are mainly focused on inferring the underlying reward function (in IRL) or imitating of the expert demonstrations \cite{ng2000algorithms,abbeel2004apprenticeship,ho2016generative,hester2017learning}. Hence, most methods can only exploit demonstrations that are optimal. However, the very optimal demonstrations are hard to obtain in practice since it is known that humans often perform suboptimal behaviors. Therefore, mediocre and unsuccessful demonstrations have long been neglected or even expelled in RL. In this paper, we show how to make use of seemingly-useless demonstrations in the exploration process to improve sample efficiency.

% The trade-off between exploration and exploitation has been a long-standing issue in the RL literature. 
Speaking of efficient exploration strategy, it expects an agent to balance between {\it exploring} poorly-understood state-action pairs to get better performance in the future and {\it exploiting} existing knowledge to get better performance now. The exploration vs exploitation problem also has two families of methods: model-based and model-free. Model-based means the agent explicitly model the Markov Decision Process (MDP) environment, then does planning over the model. In contrast, model-free methods maintain no such environment model. Typical model-free exploration approaches include $\epsilon$-greedy\cite{sutton1998reinforcement}, optimistic initialization\cite{ross2011reduction}, and more sophisticated ones such as noisy network \cite{fortunato2017noisy} and curiosity\cite{pathak2017curiosity}. These model-free exploration strategies usually are capable to handle large scale real problems, however, they do not have a theoretic guarantee. Whereas, the model-based explorations are more systematic, thus often have theoretic bounds, such as Optimism in the Face of Uncertainty (OFU)\cite{jaksch2010near} and Posterior Sampling (PS) Reinforcement Learning (PSRL)\cite{osband2013}. Despite the beautiful theoretical guarantees, the model-based methods suffer from significant computation complexity when state-action space is large, hence usually not suitable for large scale real problem. 

How can we combine the advantage of both demonstration and exploration strategy to gain an even more efficient learning for RL? In this paper, we propose a model-free RL exploration algorithm GPPSTD using posterior sampling on joint Gaussian value function, and provide theoretical analysis about its efficiency in the meantime. We also make use of various demonstrations to decrease the expectation uncertainty of Q value model, and then leverages this advantage in implementing PS on Q values to gain more efficient exploration. 

In summary our contributions include: 
\begin{itemize}
\setlength{\itemsep}{2pt}
\setlength{\parsep}{0pt}
\setlength{\parskip}{0pt}
\item Show that posterior sampling based on model-free Gaussian Process could achieve a BayesRegret Bound of $\tilde{O}( \sqrt{HT})$ with deterministic environment and bayesian cumulative error of estimation bound for a single state of $\tilde{O}( \sqrt{\lceil\frac{T}{H}\rceil})$.
\item Propose the GPPSTD algorithm to leverage posterior sampling together with various demonstration to improve the learning efficiency of RL.
\item Prove that making use of various demonstrations could decrease the expectation of GP uncertainty.
\item Show empirical results for GPPSTD exploration efficiency and an even more efficient learning when using various demonstrations.
\end{itemize}

%demonstration use for experience but not just experience
%efficiency--explore&exploit trade-off
%GP general take use of uncertainty --texun
%Thompson sampling  
%bayesian neural net work and Gaussian process-promising future work

%(GP&GPTD)
\section{Related Work}
%%demo
Two typical methods of learning from demonstration, are inverse reinforcement learning (IRL) and imitation learning (IL). Inverse reinforcement learning was introduced in \citealt{ng2000algorithms}. Its goal is to infer the underlying reward function given the optimal demonstration behavior. Further IRL algorithm includes Bayesian IRL \cite{ramachandran2007bayesian, michini2012improving}, Maximum Entropy IRL \cite{ziebart2008maximum, audiffren2015maximum}, Repeated IRL \cite{amin2017repeated}, etc. But IRL can be intractable when problem scale is large. Earlier imitation learning indicates behavior cloning, which could fail when agent encounters untrained states. Later representative IL algorithm includes Data Aggregation (DAgger) \cite{ross2011reduction}, Generative Adversarial Imitation Learning (GAIL) \citep{ho2016generative}, etc. However, their work focuses on imitating optimal demonstration, regarding mediocre and failed demonstration unusable. They also never consider exploration problem after imitating. 

% explore
As for the exploration problem, two intuitive methods, $\epsilon$-greedy\cite{sutton1998reinforcement} and Optimistic Initialization\cite{grzes2009improving}, are the earliest way to tackle this problem. $\epsilon$-greedy is to explore with a probability $\epsilon$. Optimistic Initialization initializes all Q values to $\frac{r_{max}}{1-\gamma}$, making RL visit each state at least some times. Model based method Optimism in the Face of Uncertainty (OFU) is to assign each state-action pair a biased estimate of future value and selects the action with highest estimate \cite{jaksch2010near}. Posterior sampling method has been proposed since \cite{Strens00}, involving sampling a set of values from posterior estimation and selecting the action with maximal sampled value. PSRL proposed by \citeauthor{osband2013}(\citeyear{osband2013}) does PS on the Markov Decision Process (MDP): in every episode, PSRL sample a MDP , run model-based planning algorithm and acts as if it is the true optimal policy. For finite horizon algorithms, regret bound of $O(HS\sqrt{AT})$ is achieved by PSRL  \cite{osband2013}, and $O(H\sqrt{SAT})$ by GPSRL\cite{osband17a}. It is notable that these methods are all model-based with finite SA space, which can be a considerable limitation in application. 

However, since PSRL is a model-based algorithm, it suffers from significant computation complexity for planning when state and action space are large. Therefore, in this paper we built model on value function based on Gaussian Process (GP), making it model-free, and to achieve both exploration efficiency and tractable computation complexity. 

%demo (pretrain ? texun?)

%PS RL psrl,ucb,e-e trade off method,bootstrappeddqn

Previous model-free algorithms have also been proposed using GP in RL. GP-SARSA \cite{Engel2005} used GP to update posterior estimation of value function by temporal difference method. iGP-SARSA proposed informative exploration but lacks theoretical analysis \cite{chung2013}. GPQ for both on-line and batch settings aims at learning Q function which could actually converge as $T \rightarrow \infty$ \cite{chowdhary2014} but lacks efficient exploration. DGPQ employed delayed update of Q function to achieve PAC-MDP\cite{grande14} but still lacks efficient exploration.

For regret bounds under GP hypothesis, \citeauthor{srinivas2012}\yrcite{srinivas2012} used GP to analyze the regret bound using information gain in bandit problems, while posterior sampling using GP and related analysis of regret bounds had not been explored yet, which would be discussed in this paper.
%GP RL gpq-off policy,delayed gpq

\section{Theoretical Analysis}
In this section,  we will show that how to choose demonstrations to achieve lower expected estimation variance, analyze related bounds of posterior sampling in RL under the hypothesis of GP for both deterministic and non-deterministic MDPs, and finally relate the choice of demonstrations and posterior sampling for efficiency improvement.
%, and propose our algorithm (GPPSTD) to realize our idea. Finally, we discuss further extension of bayesian neural networks in RL based on the relationship between GP and BNN.

\label{Analysis}
\subsection{Expectation of variance conditioned on data in GP}
\label{sec:expectation_variance}

We choose joint Gaussian distribution on value function \--- more specifically, Gaussian Process (GP) \--- because GP provides a principled, practical, probabilistic approach to learn in kernel machines\cite{Rasmussen2006}.

We assume that the values in the value function are joint normal distributed. Under the GP assumption, the posterior distribution are given by 
\begin{small}
\begin{sequation}
\begin{split}
f^*|X^*,X,f \backsim \mathcal{N}(K(X^*,X)K(X,X)^{-1}f, \\K(X^*,X^*)-K(X^*,X)K(X,X)^{-1}K(X,X^*)),
\end{split}
\end{sequation}
\end{small}

where $f$ is the value of the state vector $X$, and we wish to obtain value estimation $f^*$ over the new observation $X^*$. $f$ and $X$ come from history or what we call experiences. We define $p(x)$ as the distribution of test points, i.e. the states which occur in RL. In the framework of RL, $x$ is every single state and its visiting distribution $p(x)$ is determined by current policy $\mu$ and the MDP (Markov Decision Process).

We will start by a theorem that is quite obvious from intuition but hasn't been proved yet.

\begin{theorem}  
When a set $(X', f)$ is used to estimate $f(x^*)$ in GP, the expectation of variance on test points $x^*$ with distribution $p(x)$ conditioned on all possible training set (X', f) set would not be less than what conditioned on the training set $X$ sampled from distribution $p(x)$, if the size of sample set is large enough to ignore the approximation error.
\begin{sequation}
\begin{split}
& \int \{K(x^*,x^*)-K(x^*,X)K(X,X)^{-1}K(X,x^*))\}p(x)dx\\
& \leq \int \{K(x^*,x^*)-K(x^*,X')K(X',X')^{-1}K(X',x^*))\}p(x)dx
\end{split}
\end{sequation}
\end{theorem}

\begin{proof}  

Given a kernel $K$, together with a distribution $p(x)$, there is a corresponding series of eigenfunctions $\phi(x)$, s.t. $\int k(x,x')\phi(x)d\mu(x)=\int k(x,x')\phi(x)p(x)dx=\lambda\phi(x')$, and $\forall i, \int\phi_i(x)\phi_j^*(x)p(x)dx=\delta_{ij}$(* here means conjugate transpose) \cite{Rasmussen2006}.

We consider the expectation of posterior variance over the distribution $p(x)$ given any $X'$ as
$\int \{K(x^*,x^*)-K(x^*,X')K(X',X')^{-1}K(X',x^*))\}p(x)dx$. 

Since $\int K(x^*,x^*)p(x)dx$ has no relation with $X'$, we just focus on the  latter subtracted part $\int K(x,X')K(X',X')^{-1}K(X',x)p(x)d(x)$.
According to Mercer's theorem, $K(x,x')=\Sigma_{i=1}^\infty \lambda_i\phi_i(x)\phi_i^*(x')$. 
\begin{small}
\begin{sequation}
\begin{split}
\int K(x,X')K(X',X')^{-1}K(X',x)p(x)d(x)\\
=\int \{(\Sigma_{i=1}^\infty \lambda_i\phi_i(x)\phi_i^*(X'))K(X',X')^{-1}\\
(\Sigma_{j=1}^\infty \lambda_j\phi_j(X')\phi_j^*(x))\}p(x)dx.
\end{split}
\end{sequation}
\end{small}

If $i$ does not equals to $j$, the integral would be $0$. 

So 
\begin{small}
\begin{equation*}
\begin{split}
&\int \{\Sigma_{i=1}^\infty \lambda_i\phi_i(x)\phi_i^*(X')K(X',X')^{-1}\lambda_i\phi_i(X')\phi_i^*(x)\}p(x)dx\\
&=\Sigma_{i=1}^\infty \lambda_i^2\phi_i^*(X')K(X',X')^{-1}\phi_i(X').
\end{split}
\end{equation*}
\end{small}

For each i, focus on $\phi_i^*(X')K(X',X')^{-1}\phi_i(X')$.

Using numerical approximation of eigenfunctions \cite{Rasmussen2006}, when each $x_l$ is sampled from the distribution $p(x)$,
$\lambda_i\phi_i(x)=\int k(x,x')p(x)\phi_i(x)\backsimeq \frac{1}{n}\Sigma_{l=1}^nk(x_l,x')\phi_i(x_l)$. Plugging in $x'=x_l$, we get $K(X,X)u_i=\lambda_i^{mat}u_i$, where $X=[x_l]$ and  $K_{i,j}=k(x_i,x_j)$,and each $u_i$ and $\lambda_i^{mat}$ is the eigenvector and eigenvalue of matrix $K(X,X)$, with the approximation $\phi_i(X)\backsimeq \sqrt n u_i, \frac{1}{n}\lambda_i^{mat}\backsimeq\lambda_i$.

Given a random set of $X'$, and a sampled set of $X$, although we do not know $\phi(x)$ exactly,  we can use $X$ to estimate the value of eigenfunctions of $X'$: $\phi_i(X')\backsimeq\frac{\sqrt n}{\lambda_i^{mat}}K(X',X)u_i$.

Now that
\begin{equation*}
\begin{split}
&\phi_i^*(X')K(X',X')^{-1}\phi_i(X')\\
&\backsimeq \frac{n}{(\lambda_i^{mat})^2}u_i^TK(X,X')K(X',X')^{-1}K(X',X)u_i
\end{split}
\end{equation*}
, when $n \rightarrow \infty$ we could regard all above estimations as asymptotic unbiased estimations, and here we suppose n is large enough to ignore the approximation error so the approximate equations can be seen as equations.

Applying matrix decomposition to symmetric non-negative definite matrix $K(X',X')$ ,

$K(X',X')^{-1}=\Sigma_{j=1}^n \frac{1}{ \lambda_j^{'mat}}v_jv_j^*$.

So $\frac{n}{(\lambda_i^{mat})^2}u_i^*K(X,X')K(X',X')^{-1}K(X',X)u_i=\Sigma_j\frac{n}{(\lambda_i^{mat})^2 \lambda_j^{'mat}}||u_i^*K(X,X')v_j||^2$.

On each $\frac{n}{(\lambda_i^{mat})^2 \lambda_j^{'mat}}||u_i^*K(X,X')v_j||^2$, we have:
\begin{small}
\begin{sequation}
K(X,X')=\Sigma\phi(X)\phi(X')^*=\psi(X)\psi(X')^*
\end{sequation}
\begin{sequation}
\begin{split}
&u_iK(X,X)u_i^*=\lambda_i^{mat}\\
&=u_i\psi(X)\psi(X)^*u_i^*=||u_i\psi(X)||^2
\end{split}
\end{sequation}
\begin{sequation}
\begin{split}
&v_jK(X',X')v_j^*= \lambda_j^{'mat}\\
&=v_j\psi(X')\psi(X')^Tv_j^*=||v_i\psi(X')||^2
\end{split}
\end{sequation}
\end{small}

According to Cauchy-Schwartz inequality, 

$|u_i\psi(X)\psi(X')^*v_j^*|\leqslant ||u_i\psi(X)||\quad||v_i\psi(X')||=\sqrt {\lambda_i^{mat}  \lambda_j^{'mat}}.$

So 
\begin{small}
\begin{sequation}
\begin{split}
&\frac{n}{(\lambda_i^{mat})^2 \lambda_j^{'mat}}||u_i^TK(X,X')v_j||^2\leqslant \frac{n}{(\lambda_i^{mat})^2 \lambda_j^{'mat}}\lambda_i^{mat} \lambda_j^{'mat}\\
&=\frac{n}{\lambda_i^{mat}}=\frac{1}{\lambda_i},
\end{split}
\end{sequation}
\end{small}

and when $ \lambda_j^{'mat}$ equals to $\lambda_i^{mat}$ the result can reach its largest, and the lowest expectation of overall conditional variance is $\int k(x,x)p(x)dx - \Sigma_{i=1}^\infty \lambda_i$. Especially,  when RBF kernel is selected, under any $p(x)$ the lowest expectation would be $1-\Sigma_{i=1}^\infty \lambda_i=1-\lim_{n \to \infty}\frac{\Sigma_{i=1}^n\lambda_i^{mat}}{n}=1-\lim_{n \to \infty}\frac{trace(K_{XX})}{n}=0$.

Moreover, if the kernel contains noise as below:
\begin{sequation}
\begin{split}
f^*|X^*,X,f \backsim \mathcal{N}(K(X^*,X)[K(X,X)+\sigma^2I]^{-1}f, \\K(X^*,X^*)-K(X^*,X)[K(X,X)+\sigma^2I]^{-1}K(X,X^*)),
\end{split}
\end{sequation}
$K(X,X)+\sigma^2I$ would still be a symmetric non-negative definite matrix, so the eigenvalues in previous analysis would all be added with $\sigma^2$, and the eigenvectors remain the same. Obviously the conclusion still remains the same.
\end{proof} 
Notice that during a learning process of RL, if the agent has not learned how to perform perfectly yet, under present policy the states which the agent would come across would not be those of highest real value. So non-perfect demonstrations are necessary to lower the expectation of uncertainty during exploration.

\subsection{BayesRegret of GP-based Posterior Sampling}
\label{bg}
\subsubsection{deterministic MDP}
We start with a simple case where transitions are deterministic in MDP.

We model each MDP $M=\{S,A,R^{M},P^{M},H\}\sim\phi$, with potentially infinite sets of states $S$ and actions $A$. 
$H$ is the length of a single episode. At timestep t of an episode, the agent observe $s_t \in S$, select $a_t \in A$, receive a reward $r^M_{s_t,a_t} \sim R^M(s_t,a_t)$ and transition $s_{t+1}=P^M(s_t,a_t)$. $\bar{r}^M_{s_t,a_t}=\mathbb{E}[r^M_{s_t,a_t} |r^M_{s_t,a_t} \sim R^M(s_t,a_t)].$ 

$\mu$ is the policy function of state, and value function:
$V_{\mu,M}(s)=\mathbb{E}[\Sigma_{i=0}^{\infty}\gamma^{i}\bar{r}^M_{s_{i+1},a_{i+1}}|s_{i+1}=P^M(s_i,a_i),a_i=\mu(s_i)]$, 
where $\gamma$ is the rate of discount and satisfies $0<\gamma \leq  1$.

$M^k$ is the posterior sample of unknown true MDP $M^*$ given history $\mathcal{H}_{kt}$, $\mathcal{H}_{kt}=\{s_{1,1},a_{1,1},r_{1,1},s_{1,2},......s_{k,t-1},a_{k,t-1},r_{k-1,t-1}\}$.  
$\mu^M$is the optimistic policy under $M$, $ \mu^{k} \in arg max_{\mu} V_{\mu,M^k}(s)$, and particularly,  $\mu^{k},\mu^*$is the optimistic policy under $M^{k},M^*$ separately. $\pi$ indicates the learning algorithm which choose a policy $\mu$ for the agent to perform. 

We assume that given MDP $V_{\mu^M,M}(s)$ is joint normal on the set of state $S$ with optimal policy $\mu^M$ in $M$, which contains the assumption of the model using a model-free method.

Define expected cumulative reward of the kth episode:
\begin{sequation}
\begin{split}
&S_{k}^{\mu,M}=\mathbb{E}[\Sigma_{t=H(k-1)+1}^{Hk} \bar{r}^M_{s_t,a_t}\\
&|s(t+1)=P^M(s_t,a_t),a_i=\mu(s_i)].
\end{split}
\end{sequation}

Regret :
\begin{sequation}
Regret(T,\pi,M^*)=\Sigma_{k=1}^{\lceil\frac{T}{H}\rceil} (S_k^{\mu^*,M^*}-S_k^{\mu^k,M^*}).
\end{sequation}

The regret of every episode is random due to the unknown true MDP $M^*$, the learning algorithm $\pi$, the sampling $M^k$ of the present episode and previous sampling through history $\mathcal{H}_{k1}$. Notice that in our algorithm we do not directly sample $M^k$ from the posterior distribution $\phi(\cdot|\mathcal{H}_{k1})$ and we use the posterior distribution of the value to realize our sampling. But for convenience we would use sampled $M^k$ to refer to our way of sampling in practice.

And Bayesian regret:
\begin{sequation}
\begin{split}
&BayesRegret(T,\pi,\phi)=\\
&\mathbb{E}[\Sigma_{k=1}^{\lceil\frac{T}{H}\rceil} (S_k^{\mu^*,M^*}-S_k^{\mu^k,M^*})|M^*\sim \phi].
\end{split}
\end{sequation}

which is actually the same with the regret defined by Osband \& Van Roy\yrcite{osband17a}. Since we have different definition of the value function, we use other notations to avoid confusion. 

We separate this BayesRegret by episodes, where each episode k conditioned on the previous history $\mathcal{H}_{k1}$ then taking expectation again in order to achieve the expectation on $M^*$. We discuss the relation between BayesRegret and the conditional regret, which is different from the method of previous work \cite{osband17a}.  The conditional regret is $\mathbb{E}[S_k^{\mu^*,M^*}-S_k^{\mu^k,M^*}|M^* \sim \phi(\cdot|\mathcal{H}_{k1})]$, and $\Sigma_{k=1}^{\lceil\frac{T}{H}\rceil}\mathbb{E}[S_k^{\mu^*,M^*}-S_k^{\mu^k,M^*}|M^* \sim \phi]=\mathbb{E}[\Sigma_{k=1}^{\lceil\frac{T}{H}\rceil}\{\mathbb{E}[S_k^{\mu^*,M^*}-S_k^{\mu^k,M^*}|M^* \sim \phi(\cdot|\mathcal{H}_{k1})]| \mathcal{H}_{k1} \sim PreviousSampling\}]$. It is obvious that each $\mathcal{H}_{k1}$ here contains previous history  and not independent. So when we take expectation of a series of $\mathcal{H}_{k1}$, we actually take expectation on whole history $\mathcal{H}$.  

%We define $f_k(\mathcal{H})=\mathcal{H}_{k1}$,
%So 
%\begin{sequation}
%\begin{split}
%& \Sigma_{k=1}^{\lceil\frac{T}{H}\rceil}\mathbb{E}[S_k^{\mu^*,M^*}-S_k^{\mu^k,M^*}|M^* \sim \phi()]\\
%&=\mathbb{E}[\Sigma_{k=1}^{\lceil\frac{T}{H}\rceil}\mathbb{E}[S_k^{\mu^*,M^*}-S_k^{\mu^k,M^*}|M^* \sim \phi(\cdot| f_k(\mathcal{H}))]| \mathcal{H} \sim p()]\\
%&=\mathbb{E}\Sigma_{k=1}^{\lceil\frac{T}{H}\rceil}[\mathbb{E}[S_k^{\mu^*,M^*}-S_k^{\mu^k,M^*}|M^* \sim \phi(\cdot|\mathcal{H}_{k1})]| \mathcal{H}\sim p()].
%\end{split}
%\end{sequation}

Since we will use the stochastic property of $M^k$ to analyze the bound, another thing to notice is that since $\mathcal{H}$ is actually produced by every sampled $M^k$, taking expectation of $\mathcal{H}$ would not disturb the distribution of $M^k$. So if we can bound conditional regret (as described above) on every possible $M^k$ from its distribution, then taking the expectation would also bound BayesRegret.

\begin{theorem}  
Let $M^*$ be the true MDP with deterministic transitions according to prior $\phi$ with values under GP hypothesis. Then the regret for GPPSTD is bounded:
$BayesRegret(T,\pi^{GPPSTD},\phi)=  \tilde{O}(\sqrt{HT}).$
\end{theorem}  
\begin{proof}  

Decomposition:
\begin{sequation}
S_k^{\mu^*,M^*}-S_k^{\mu^k,M^*}=(S_k^{\mu^*,M^*}-S_k^{\mu^k,M^k})+(S_k^{\mu^k,M^k}-S_k^{\mu^k,M^*}).
\end{sequation}

First we focus on the difference we can observe by the policy $\mu^k$ that the agent actually follows, i.e. $S_k^{\mu^k,M^k}-S_k^{\mu^k,M^*}$ in (12).

Referring to previous defination (9),

\begin{sequation}
\begin{split}
&\mathbb{E}_{M^*}[S_k^{\mu^k,M^k}-S_k^{\mu^k,M^*}|\mathcal{H}_{k1}]=V_{\mu^k,M^k}(s_1)\\
&+\Sigma_{t=2}^{H-1}(1-\gamma)V_{\mu^k,M^k}(s_t)-\gamma V_{\mu^k,M^k}(s_H)-\mathbb{E}_{M^*}[V_{\mu^k,M^*}(s_1)\\
&+\Sigma_{t=2}^{H-1}(1-\gamma)V_{\mu^k,M^*}(s'_t)-\gamma V_{\mu^k,M^*}(s'_H)|\mathcal{H}_{k1}],
\end{split}
\end{sequation}

where $\mathbb{E}_{M^*}[\quad|\mathcal{H}_{k1}]$ means taking expectation on $M^* \sim \phi(\cdot|\mathcal{H}_{k1})$. 
Recall our assumption of V in the definition part. Although $V_{\mu^k,M^*}$ does not satisfy joint normal distribution since its policy is not optimistic of its MDP,  $M^k$ is still sampled from the posterior distribution of $M^*$, which means that given history $\mathcal{H}_{k1}$, the posterior sample $R^k$, $P^k$ and unknown true $R^*$, $P^*$ are identically distributed. So the expectation of (13) (on $M^k$ while performing posterior sampling) is zero. So $\mathbb{E}_{M^*}[S_k^{\mu^k,M^k}-S_k^{\mu^k,M^*}|\mathcal{H}_{k1}]$ is totally zero-mean, and is a sum of a series of joint normal variables. We would focus on the variance next.

\begin{lemma}  
(Transformation of Joint Normal Variables).

If $X\sim N_p(\mu,\Sigma)$, A is a matrix of $l\times p$ and $rank(A)=l$, $Y=AX+b$, Then

$$Y\sim N_l(A\mu+b,A\Sigma A^T).$$
\end{lemma}  

To calculate the sum , let A be a vector filled with 1, so we have $\Sigma_{i=1}^n X_i \sim (N(\Sigma_{i=1}^n\mu_i,\Sigma_{i=1}^n\Sigma_{j=1}^{n}Cov(X_i,X_j)))$

Noticing that $Cov(X,Y)=\mathbb{E}[(X-\mathbb{E}[X])(Y-\mathbb{E}[Y])] $

$\leq \sqrt{\mathbb{E}[(X-\mathbb{E}[X])^2]\mathbb{E}[(Y-\mathbb{E}[Y])^2]} =\sigma_1 *\sigma_2 \leq _{max} \sigma^2$.

So we have proved that given history $\mathcal{H}_{k1}$, $\mathbb{E}_{M^*}[S_k^{\mu^k,M^k}-S_k^{\mu^k,M^*}|\mathcal{H}_{k1}]$ is normally distributed with expectation of 0 and variance $\leq H^2_{max(k)}\sigma^2$, where $_{max(k)}\sigma^2$ is the max variance of every state in episode k.

Now back to the first difference of (12).

\begin{lemma}  (Posterior Sampling).

If $\phi$ is the distribution of $M^*$, then for any  $\sigma(_{\mathcal{H}k1})$-measurable function g,

$$ \mathbb{E}[g(M^*)|\mathcal{H}_{k1}]=\mathbb{E}[g(M^k)|\mathcal{H}_{k1}].$$
\end{lemma}  

Using the posterior lemma, the cumulative reward $S_k^{\mu^M,M}$ is $\sigma(H_{k1})$-measurable, so $ \mathbb{E}[S_k^{\mu^*,M^*}-S_k^{\mu^k,M^k}|\mathcal{H}_{k1}]=0$ \cite{osband2013}.  

Recall that $S_{k}^{\mu,M}$ is the sum of joint normal variables,  so similar to previous analysis, each $\mathbb{E}_{M^*}[S_k^{\mu^k,M^k}-S_k^{\mu^*,M^*}|\mathcal{H}_{k1}]$ is normally distributed with zero-mean and variance $\leq H^2_{max(k)}\sigma^2$.

So $\mathbb{E}_{M^*}[(S_k^{\mu^*,M^*}-S_k^{\mu^k,M^k})+(S_k^{\mu^k,M^k}-S_k^{\mu^k,M^*})|\mathcal{H}_{k1}]$ has zero-mean and variance $\leq 4H^2_{max}\sigma^2$ by analyzing covariance as previous part.

For normal distribution $X \sim \mathcal{N}(0,\sigma^2)$, and for any $1>\delta>0$, 
$\mathbb{P}(X\leq\sqrt{-2\sigma^2log\delta})\geq 1-\delta$,
which means there is a probability of $1-\delta$ that 
$X \leq \sqrt{-2  \sigma^2log\delta}$.

So noticing the independence of sampling between episodes, calculate $\mathbb{E}_{\mathcal{H}}[\Sigma_{k=1}^{\lceil\frac{T}{H}\rceil} (S_k^{\mu^*,M^*}-S_k^{\mu^k,M^k})|\mathcal{H}_{k1}]$ as analyzed before, where $\mathbb{E}_{\mathcal{H}}$ means taking expectation on $\mathcal{H}$. Set $\delta$ as $\frac{1}{T}$, and let $_{max}\sigma^2$ be the max variance of all states in all episodes (just for worst case bound), and there is a probability of $1-\frac{1}{T}$ that:

$\mathbb{E}[\Sigma_{k=1}^{\lceil\frac{T}{H}\rceil}(S_k^{\mu^k,M^*}-S_k^{\mu^k,M^*})|M^* \sim \phi]$

$\leq 2\sqrt{2_{max}\sigma^2(HT+H)logT}.$

\end{proof} 

In general cases (like RBF and Matern), $\sigma^2$ is bounded (in a few cases like dot-product kernels, covariance cannot be bounded only in infinite spaces, while most continuous spaces in RL has borders), so this could be a sub-linear bound which means the agent would actually learn the real MDP in the end. Notice that we use $_{max}\sigma^2$ only for a worst case bound in brief, while the true regret is related with each variance and covariance of the state. This result is better than previous posterior sampling analysis (PSRL bounds $\sqrt{HSAT}$ empirically but $H\sqrt{SAT}$ theoretically). As GP gets more information of the environment during exploration, the variance would decay, so actually the bound could be even better.

\subsubsection{non-deterministic MDP}

True MDP $M^*=\{S,A,R^{M},P^M,H,\rho\}\sim\phi$, other notations are just the same as 3.2.1, 
except that $P^M$ is a stochastic transition in $M$,
$\rho$ is the distribution of initial states.

Since the transition is not deterministic and the states are continuous, the cumulative reward could be related to countless states of values. Since we do not have assumptions on stochastic transition function, which is necessary for regret analysis in non-deterministic environment, we focus on the cumulative estimation error for any single state during the learning process.
\begin{sequation}
CumError(T,\pi,M^*,s)=\Sigma_{k=1}^{\lceil\frac{T}{H}\rceil} (V_k^{\mu^*,M^*}(s)-V_k^{\mu^k,M^*}(s)).
\end{sequation}

We would show that $CumError$ can also lead to the convergence of estimation as described below. We put the proof of Theorem 3 in Appendix A.
\begin{theorem}  
Let $M^*$ be the true MDP with non-deterministic transitions according to prior $\phi$ with with values under GP hypothesis, we have the bayesian cumulative error of estimation of any single state $s$:
$\mathbb{E}[CumError(T,\pi,M^*,s)|M^*\sim \phi]=\tilde{O}( \sqrt{\lceil\frac{T}{H}\rceil})$.
And let $\mathcal{M}$ be any family of MDPs with non-zero probability under the prior $\phi$. Then for any $\epsilon \geq 0$:

$\mathbb{P}(\frac{CumError(T,\pi,M^*,s)}{T}\geq \epsilon|M^*\in \mathcal{M}) \rightarrow 0.$
\end{theorem}

\subsection{Demonstrations for Posterior Samping}

Now back to our reason to make use of demonstrations. Consider the expected variance of all states with distribution $p(s)$ of our estimate of value function, where $p(s)$ is determined by posterior distribution of value function and the present policy. The analysis  in 3.2.1\&3.2.2 use $_{max}\sigma^2$ only for a worst bound, while the real situation is determined by every single $\sigma^2$. So if we get lower expected variance, lower regret would be achieved with a high probability by Markov's inequality:  $\mathbb{P}(\sigma^2\ \geq a)\leq \frac{\mathbb{E}[\sigma^2]}{a}$.  That is, with the same parameter $a$, the lower the expectation is, there is a lower probability that $\sigma^2$ would be larger than $a$. 

Above analysis requires that we use sample set $X$ which from distribution $p(x)$ as demonstrations, while in fact we do not know the exact $p(x)$. So as a 
compromise, we could improve the efficiency of our learning process by demonstrations that contains similar situations to present episode, which is rational from intuition, and also produce better result in practice in Section \ref{sec:experiment}.

\section{Gaussian Process for Posterior Sampling}
\subsection{Gaussian Process Temporal Difference}
\label{sec:GPTD}
GPTD was firstly introduced in \citealt{engel2003bayes}, then improved in \citealt{Engel2005}. We'll briefly explain its overview framework here since our algorithm is closely related to it. 

GPTD proposes a generative model for the sequence of rewards corresponding to the trajectory $x_1, x_2, \cdots, x_t$: 
\begin{sequation}
\label{eq:GPTDRVN}
R(x_i, x_{i+1}) = V(x_i) - \gamma V(x_{i+1}) + N(x_i, x_{i+1})
\end{sequation}
where $R$ is the reward process observed in experience, $V$ is the value Gaussian process, and $N$ is a noise process. 

Define 
\begin{sequation}
\label{eq:GPTDH}
H_t =  \left[
 \begin{matrix}
   1 & -\gamma & 0 & \cdots & 0 \\
   0 & 1 & -\gamma & \cdots & 0 \\
   \vdots & & & & \vdots \\
   0 & 0 & \cdots & 1 & -\gamma \\
  \end{matrix}
  \right]
\end{sequation}
We may rewrite (\ref{eq:GPTDRVN}) using (\ref{eq:GPTDH}) as
\begin{sequation}
R_{t-1} = H_tV_t + N_t
\end{sequation}

In order to complete the probabilistic generative model connecting reward observations and values, we may impose a Gaussian prior over $V$, i.e. $ V \sim \mathcal{N}(0, k(\cdot, \cdot))$, in which $k$ is the kernel chosen to reflect our prior beliefs concerning the correlations between the values. We also need to define $N_t \sim \mathcal{N}(\mathbf{0}, \Sigma_t)$ with $\Sigma_t = \sigma^2H_tH_t^T$ and $\sigma$ is the observation noise level\cite{Engel2005}. 

Since both the value prior and the observation noise are Gaussian, the posterior distribution of the value conditioned on observation sequence $\mathbf{r}_{t-1} = {(r_0, \cdots, r_{t-1})}^T$ are also Gaussian and given by
\begin{sequation}
\begin{split}
\label{eq:GPTD}
&\hat{v}_t(x) = {\mathbf{k}_t(x)}^T\bm{\alpha}_t \\
&p_t(x) = k(x, x) - {\mathbf{k}_t(x)}^T\mathbf{C}_tk_t(x)\\
where \quad &\mathbf{k}_t(x) = {(k(x_0, x), \cdots, k(x_t, x))}^T\\
& \mathbf{K}_t = \left[
 \begin{matrix}
   k(x_0, x_0) & k(x_0, x_1) & \cdots & k(x_0, x_t) \\
   k(x_1, x_0) & k(x_1, x_1) & \cdots & k(x_1, x_t) \\
   \vdots & \vdots & & \vdots \\
   k(x_t, x_0) & k(x_t, x_1) & \cdots & k(x_t, x_t) \\
  \end{matrix}
  \right] \\
&\bm{\alpha}_t = \mathbf{H}_t^T{(\mathbf{H}_t\mathbf{K}_t\mathbf{H}_t^T + \Sigma_t)}^{-1}\mathbf{r}_{t-1}\\
&\mathbf{C}_t = \mathbf{H}_t^T{(\mathbf{H}_t\mathbf{K}_t\mathbf{H}_t^T + \Sigma_t)}^{-1}\mathbf{H}_t\\
\end{split}
\end{sequation}

\subsection{GPPSTD}
\label{sec:GPPSTD}
Now we are ready to present Gaussian Process Posterior Sampling Temporal Difference (GPPSTD) algorithm, described in Algorithm \ref{alg:GPPSTD}. We adopt the GPTD framework to gain the posterior Q value distribution of state action pair conditioned on all reward experiences by Equation \ref{eq:GPTD}. We note that similar to GPSARSA method\cite{Engel2005}, we treat state action pair as $x_t$, therefore model Q value of state action pair rather than V value of state in GP. We also use episodic algorithm with fixed episode length as required by the analysis.

As analyzed before, we only update GP model after one episode ends. Posterior sampling should depend on the joint distribution of all the state-action pair in one episode. But during the exploration, the agent would not know exactly what state-action pair it would come across in the following steps within the episode. We overcome this problem by using conditional distribution of joint variables as the analysis below. 

We applied posterior sampling method by $a = \argmax_a Q_{sampled}(s_t, a)$. Denote the already sampled $Q_i = Q(s_i, a_i) (i = 1, 2, \cdots, t)$. In a single episode, when $Q_1, Q_2, \cdots, Q_{t-1}$ have been sampled, posterior $Q_t$ and all previous $Q$ are joint Gaussian distributed, 
\begin{equation}
\begin{split}
\left[
\begin{matrix}
&\mathbf{Q}_{1\cdots t-1} \\
&\mathbf{Q}(s_t, \cdot) \\
\end{matrix}
\right] \sim 
 \mathcal{N}(\left[
 \begin{matrix}
   & \bm{\mu}_{1\cdots t-1} \\
   & \bm{\mu}_t(s_t, \cdot) \\
  \end{matrix}
 \right], 
 \left[
 \begin{matrix}
   & \bm{\Sigma}_{xx} & \bm{\Sigma}_{x^*x} \\
   & \bm{\Sigma}_{xx^*} & \bm{\Sigma}_{x^*x^*} \\
  \end{matrix}
 \right])
 \\
\end{split}
\end{equation}
in which $\mathbf{Q}(s_t, \cdot)$ stands for Q values of all actions possibilities in $s_t$, $\bm{\mu}(s_t, \cdot)$ stands for their posterior means, and $\bm{\Sigma}_{xx}$, $\bm{\Sigma}_{x^*x}$, $\bm{\Sigma_{x^*x^*}}$ stands for posterior covariance matrix given by GP. Using standard multivariate Gaussian conditional results, we gain posterior sampling 

\begin{sequation}
\label{eq:Qsample}
\begin{split}
\mathbf{Q}(s_t, \cdot) \sim \mathcal{N}(&\bm{\mu}_t(s_t, \cdot)+\bm{\Sigma}_{x^*x} \bm{\Sigma}_{xx}^{-1}(\mathbf{Q}_{1\cdots t-1}-\bm{\mu}_{1\cdots t-1}),\\
&\bm{\Sigma}_{x^*x^*}-\bm{\Sigma}_{x^*x}\bm{\Sigma}_{xx} ^{-1}\bm{\Sigma}_{xx^*})
\end{split}
\end{sequation}

By subtracting $\bm{\mu}_t$ in (\ref{eq:Qsample}), we have each conditional noise 
\begin{sequation}
\label{eq:nsample}
\begin{split}
n(s_t, a_t) \sim \mathcal{N}(&\bm{\Sigma}_{x^*x} \bm{\Sigma}_{xx}^{-1}\bm{n}_{1\cdots t-1},\\
&\bm{\Sigma}_{x^*x^*}-\bm{\Sigma}_{x^*x}\bm{\Sigma}_{xx} ^{-1}\bm{\Sigma}_{xx^*}))
\end{split}
\end{sequation}

\begin{figure*}[t]
\vskip 0.2in
\begin{center}
\centerline{\includegraphics[width=\columnwidth]{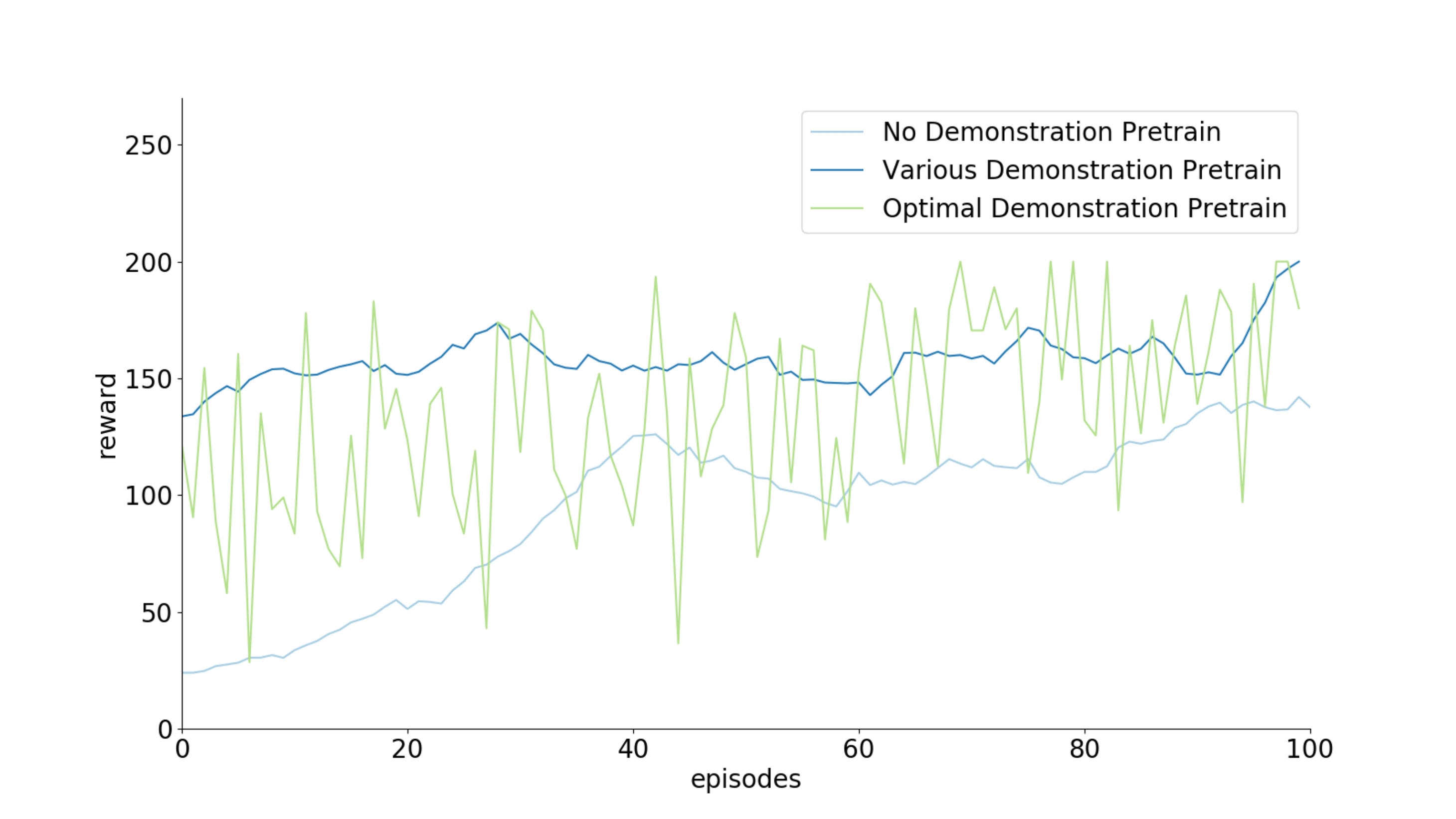}\includegraphics[width=\columnwidth]{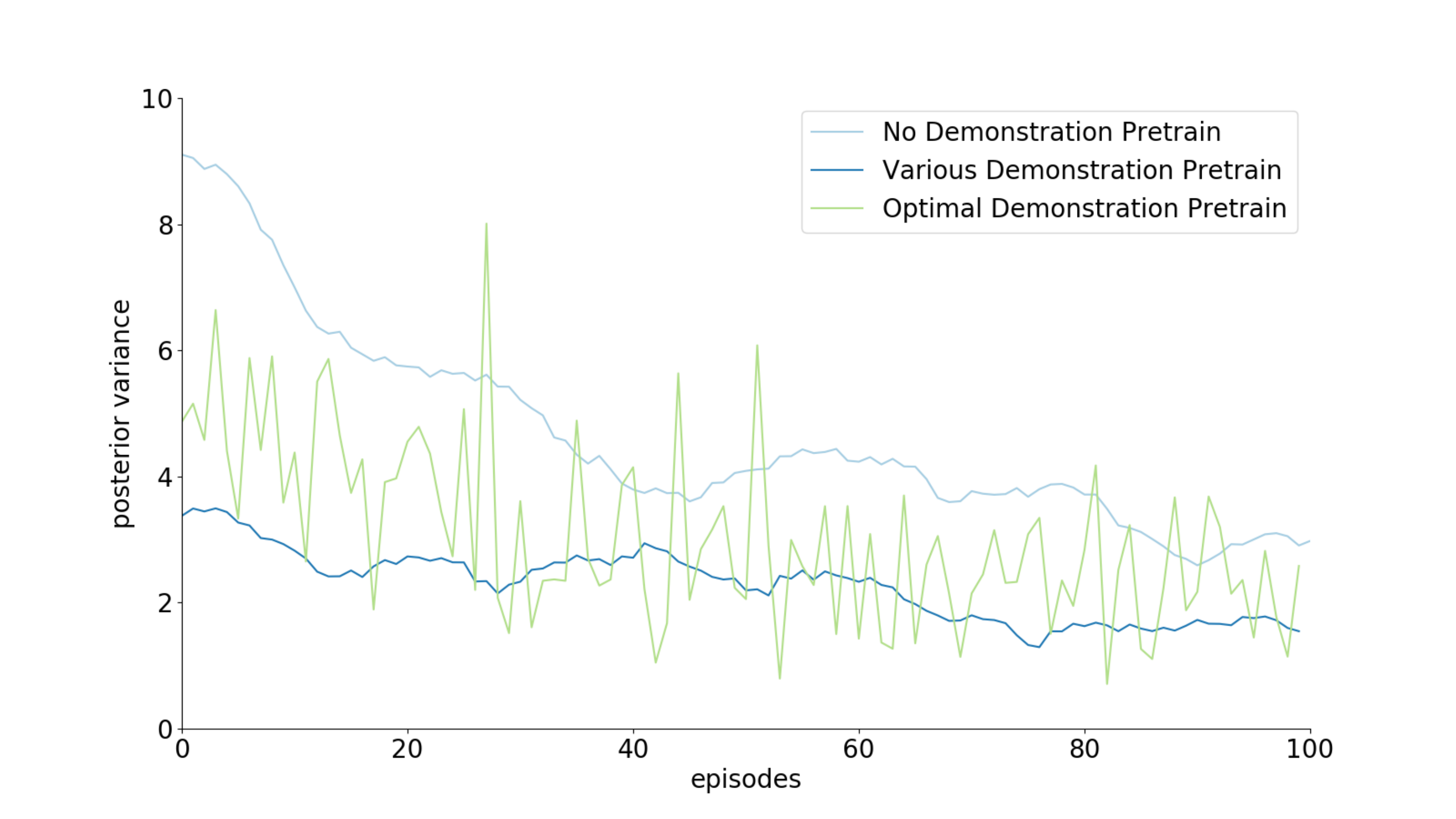}}
\caption{Performance and Variance comparison between no-pretrain, optimal-pretrain and various-pretrain settings. The results are the average of 5 experiments}
\label{fig:pretrain}
\end{center}
\vskip -0.2in
\end{figure*}

So at each timestep t, we perform action selection by sampling a noise from conditional distribution, add it to posterior mean of $Q$ and choose the best action according to the noised $Q$. At the end of the episode we use collected observation sequence to update our GP model by updating $\bm{K}_t, \bm{\alpha}_t, \bm{C}_t$ in (\ref{eq:GPTD})(for more detail, we refer readers to \citealt{Engel2005}). This exploration is bounded by Theorem 2 (in deterministic environments) or Theorem 3 (with non-deterministic environments).

It is worth mentioning that because our policy remain unchanged during one episode, it achieves deep exploration. \cite{russo2017tutorial}.

\begin{algorithm}[tb]
   \caption{GPPSTD}
   \label{alg:GPPSTD}
\begin{algorithmic}
   \STATE {\bfseries Initialize} GP model $M$
   \REPEAT
   \STATE {\bfseries Initialize} initial state $s_1$, Memory of the episode
   \FOR{{\bfseries timestep} $t=1$ {\bfseries to} $H$}
   \STATE Obtain $\mu(s_t, \cdot)$,$\Sigma$ from $M$ using \ref{eq:GPTD}
   \STATE Sample $n(s_t, \cdot)$ according to \ref{eq:nsample}
   \STATE Perform $a=\argmax_a(\mu(s_t, a) + n(s_t, a))$
   \STATE Observe $s_{t+1}$,$r$
   \STATE Memory.add($(s_{t-1}, a_{t-1}, r, s_{t}, a_{t})$
   \ENDFOR
   \STATE GPTD.Update(M, Memory)
   \UNTIL {M convergence requirement satisfied}
\end{algorithmic}
\end{algorithm}

\subsection{Pretrain}
Now let's see how we can make use of various demonstrations to make GPPSTD more efficient. The way we pretrain GP model $M$ is exactly the same as training. For RL, the "test" point distribution $p(x)$ is the experience collected in environment, which is determined by its current knowledge (in our case, value) and exploration strategy. According to analysis in Section \ref{sec:expectation_variance}, a training set sampled from $p(x)$ could give the lowest expected uncertainty, then help to avoid GPPSTD algorithm from meaningless exploration, resulting in the efficiency bound in Section \ref{bg}.

Intuitively, we could regard the various-pretrain as an sketch overview of the Q value over state action space, and this sketch helps RL agent explore smartly. Though we just pretrain data with training method, we note that it is extremely hard for the agent to obtain the sketch alone, since a large proportion of space can't be accessed by RL agent itself for lack of systematic information especially in the beginning of the training. 

\begin{figure}[t]
\vskip 0.2in
\begin{center}
\centerline{\includegraphics[width=\columnwidth]{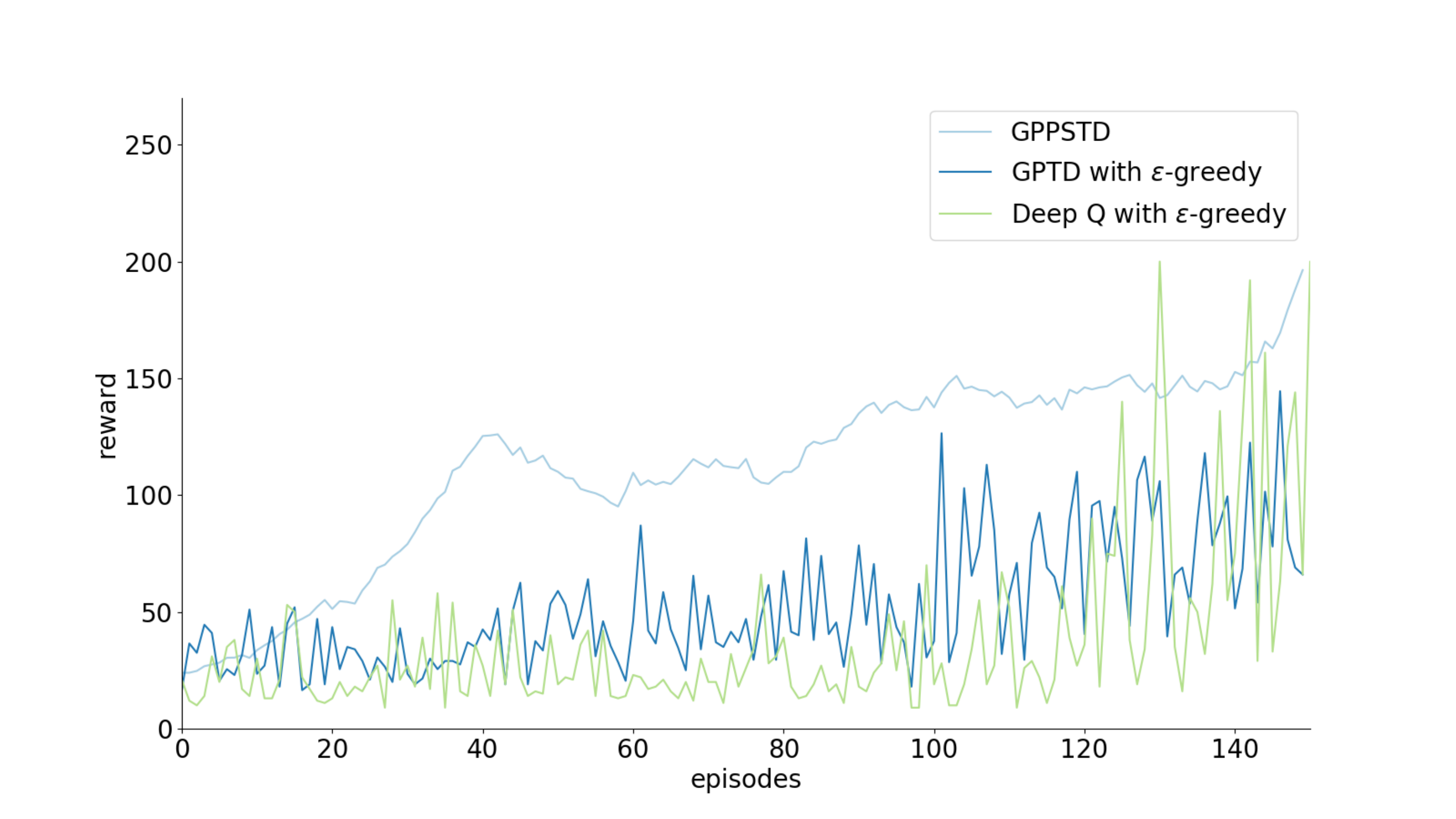}}
\caption{Performance comparison between GPPSTD, GPTD with $\epsilon$-greedy and Deep Q with $\epsilon$-greedy. The results are the average of 5 experiments}
\label{fig:nopretrain}
\end{center}
\vskip -0.2in
\end{figure}

\section{Gaussian Process and Bayesian Neural Network}
\label{bnn}
Now we'll be ready to discuss the general relationship between GP and bayesian neural networks, expanding our ideas to BNN.
Neal \yrcite{Neal96} had shown that Bayesian neural networks with infinitely many hidden units converged to Gaussian Process with a particular kernel (covariance) function. Recently, \citeauthor{lee2018} \yrcite{lee2018} has proposed NNGP to perform Bayesian prediction with a deep neural network which could outperform standard neural networks trained with stochastic gradient descent. \citeauthor{g2018} \yrcite{g2018} exhibited situations where existing Bayesian deep networks are close to Gaussian Processes. 

So based on earlier work, we could expect that our theory about efficient exploration and making use of demonstrations in RL could extend to Bayesian deep networks. Related work had been done by \citeauthor{azizzadenesheli2018} \yrcite{azizzadenesheli2018}. They proposed Bayesian Deep Q-Network  (BDQN), a practical Thompson sampling based RL Algorithm using Bayesian regression to estimate the posterior over Q-functions, and has achieved impressive results while lacks theoretical analysis. We think this paper could provide a possible theoretical justification for BDQN, meanwhile making use of demonstrations remains future work.

\section{Experiments}
\label{sec:experiment}

Our empirical experiment is done in the CartPole Task, a classic control problem in OpenAI Gym \cite{1606.01540}. The task is to push a car left or right to balance a stick on the car. In each timestep, the RL algorithm receives a 4-dimensional state, takes one of two actions (left or right), and receives a reward of 1 if the stick's deviation angle from vertical line is within a range. If not, the episode will end. The maximum length of an episode is 200 steps, and we could view the steps after failure as reward = 0, therefore making it a fixed length task. 

Firstly, we compare the performance of GPPSTD algorithm, GPTD using $\epsilon$-greedy and deep-q learning using $\epsilon$-greedy on CartPole in Fig. \ref{fig:nopretrain}. We choose squared exponential kernel $k(x_i, x_j) = c \times exp(-\frac{1}{2}{d(x_i/l, x_j/l)}^2)$ for GPPSTD and GPTD method, with length scale $l = [0.1, 0.02, 0.1, 0.02, 0.001]$ and variance $c = 10$. Since we regard state-action pair as $x$ in GP, our length scale is a 5-dimensional vector. We note that because we believe there are no value correlations in action, we give it an length scale of 0.001, which in turn will cause $k(x_i, x_j) = 0$ when action is different. Result Fig. \ref{fig:nopretrain} shows that GPPSTD significantly outperform other two algorithms. It demonstrates GPPSTD's exploration process to be both efficient and robust, since $\epsilon$-greedy methods fluctuate a lot relative to GPPSTD. We also see that GP may be a better model than neural network in this task. 

In the second experiment, we show that when combined with demonstration, GPPSTD could achieve an even better results. In the optimal demonstration pretrain setting we use 10 episodes of optimal demonstration (200-score episodes) while in the various-pretrain setting, 5 episodes of optimal demonstration and 5 episodes of unsuccessful demonstration (score between 10-60) are used for pretrain. As shown in Fig. \ref{fig:pretrain}), various-pretrain outperforms optimal-pretrain and no-pretrain. We notice that optimal-pretrain suffers fluctuate performance compared to various-pretrain, which verifies our belief. It is because that optimal demonstration only can not provide agent with the information outside optimal trajectory, which leads to higher variance of estimations, whereas various demonstration has lower variance of estimation during exploration, thus lead to better regret as our analysis in Section \ref{bg}. Moreover, as in Fig. \ref{fig:pretrain}, various-pretrain has the lowest action uncertainty (measured by posterior variance) at the beginning, reflected our analysis on expected uncertainty analysis in Section \ref{sec:expectation_variance}. 

\section{Conclusions}

In this paper, we discuss how to make use of various demonstrations to improve exploration efficiency in RL and make a statistical proof from the view of GP. What is equally important is that we propose a new algorithm GPPSTD, which implements a model-free method in continuous space with efficient exploration by posterior sampling under GP hypothesis, and also behaves impressively in practice. Both two methods aim at efficient exploration in RL. More impressively, combining both could further improve the efficiency from a Bayesian view. The property of Gaussian Process has been discussed to extend these methods to neural network, and we expect faster computation and even better results using our model-free posterior sampling methods on Bayesian Neural Network.

\bibliography{example_paper}
\bibliographystyle{icml2018}

\end{document}